\documentclass[letterpaper]{article}
\usepackage{aaai16}
\usepackage{times}
\usepackage{helvet}
\usepackage{courier}
\usepackage{url}
\usepackage{amsthm}
\usepackage{amssymb}
\usepackage{longtable}
\usepackage{amsmath}
\usepackage{graphicx} 
\usepackage{graphics, color}
\usepackage{algorithm}
\usepackage{algpseudocode}
\usepackage{url}
\newfloat{algorithm}{t}{lop}
\usepackage{enumitem}
\newcommand{\numCells}{\text{numCells}}

\newcommand{\BoundedSMT}{\ensuremath{\mathsf{BoundedSMT}}}

\newcommand{\ApproxMC}{\ensuremath{\mathsf{ApproxMC}}}

\newcommand{\killthis}[1]{}

\newtheorem{theorem}{Theorem}
\newtheorem{lemma}{Lemma}

\newcommand{\prob}{\ensuremath{\mathsf{Pr}}}
\newcommand{\expect}{\ensuremath{\mathsf{E}}}
\newcommand{\var}{\ensuremath{\mathsf{V}}}

\newcommand{\mc}[1]{\ensuremath{\mathcal{#1}}}

\newcommand{\NP}{\ensuremath{\mathsf{NP}}}
\newcommand{\SAT}{\ensuremath{\mathsf{SAT}}}
\newcommand{\SMT}{\ensuremath{\mathsf{SMT}}}
\newcommand{\sharpSATTool}{\ensuremath{\mathsf{sharpSAT}}}

\newcommand{\SMTApproxMC}{\ensuremath{\mathsf{SMTApproxMC}}}

\newcommand{\SMTApproxMCCore}{\ensuremath{\mathsf{SMTApproxMCCore}}}
\newcommand{\Boolector}{\ensuremath{\mathsf{Boolector}}}

\newcommand{\CDM}{\ensuremath{\mathsf{CDM}}}

\newcommand{\extract}[3]{\ensuremath{\mathsf{extract}({#1}, {#2},{#3})}}
\newcommand{\Slice}[2]{\ensuremath{\mathbf{#1}^{(#2)}}}
\newcommand{\X}{\ensuremath{\mathbf{X}}}

\newcommand{\ceil}[1]{\ensuremath{\lceil #1 \rceil}}
\newcommand{\floor}[1]{\ensuremath{\lfloor #1 \rfloor}}

\newcommand{\satisfying}[1]{\ensuremath{R_{#1}}} %
\algnotext{EndFor}
\algnotext{EndIf}
\algnotext{EndWhile}
\algrenewcomment[1]{\(\triangleright\) #1}

\newcommand{\Z}{\ensuremath{\mathbb{Z}}}
\newcommand{\HH}{\ensuremath{\mathcal{H}}}
\everymath{\textstyle}
\title{Approximate Probabilistic Inference via Word-Level Counting
	\thanks{The author list has been sorted alphabetically by last name; this should not be used to determine the extent of authors' contributions.}
	}
\author{Supratik Chakraborty\\Indian Institute of Technology, \\Bombay
	\And Kuldeep S. Meel\\Department of Computer Science,\\Rice University \And Rakesh Mistry\\Indian Institute of Technology, \\Bombay \AND Moshe Y. Vardi\\Department of Computer Science,\\Rice University
}
\begin{document}
\maketitle

\begin{abstract}
Hashing-based model counting has emerged as a promising approach for
large-scale probabilistic inference on graphical models. A key
component of these techniques is the use of xor-based 2-universal hash
functions that operate over Boolean domains.  Many counting problems
arising in probabilistic inference are, however, naturally encoded
over finite discrete domains.  Techniques based on bit-level (or
Boolean) hash functions require these problems to be
propositionalized, making it impossible to leverage the remarkable
progress made in {\SMT} (Satisfiability Modulo Theory) solvers that can
reason directly over words (or bit-vectors).
In this work, we present the first approximate model counter that uses
word-level hashing functions, and can directly leverage the power of
sophisticated {\SMT} solvers. Empirical evaluation over an extensive
suite of benchmarks demonstrates the promise of the approach.
\end{abstract}

\section{Introduction}
Probabilistic inference on large and uncertain data sets is
increasingly being used in a wide range of applications.  It is
well-known that probabilistic inference is polynomially
inter-reducible to model counting~\cite{Roth1996}. 
In a recent line of work, it has been
shown~\cite{CMV13b,CFMSV14,EGSS13c,IMMV15} that one can strike a fine balance 
between performance and approximation guarantees for
propositional model counting, using $2$-universal hash
functions~\cite{CW77} on Boolean domains.  This has propelled the 
model-counting formulation to emerge as a promising ``assembly 
language''~\cite{BPdB15} for inferencing in probabilistic graphical models.

In a large class of probabilistic inference problems, an important
case being lifted inference on first order representations~\cite{KLiftInf12}, 
the values of variables come from finite but large (exponential in the size 
of the representation) domains.  Data values coming from such domains are naturally 
encoded as fixed-width words, where the width is logarithmic in the size of
the domain.  Conditions on observed values are, in turn, encoded as
word-level constraints, and the corresponding model-counting problem
asks one to count the number of solutions of a word-level constraint.
It is therefore natural to ask if the success of approximate
propositional model counters can be replicated at the word-level.

The balance between efficiency and strong guarantees of hashing-based
algorithms for approximate propositional model counting
crucially depends on two factors: (i)~use of XOR-based $2$-universal
bit-level hash functions, and (ii)~use of state-of-the-art
propositional satisfiability solvers, viz. CryptoMiniSAT~\cite{SNC09}, 
that can efficiently reason about formulas that combine disjunctive clauses
with XOR clauses.  

In recent years, the performance of {\SMT} (Satisfiability Modulo
Theories) solvers has witnessed spectacular
improvements~\cite{BDdMOS12}.  Indeed, several highly optimized \SMT
solvers for fixed-width words are now available in the public
domain~\cite{BB09,JLS09,HBJBT14,DB08}.
Nevertheless, $2$-universal hash functions for fixed-width words
that are also amenable to efficient reasoning by {\SMT} solvers have
hitherto not been studied.  The reasoning power of \SMT solvers for
fixed-width words has therefore remained untapped for word-level model
counting.  Thus, it is not surprising that all existing work on
probabilistic inference using model
counting
(viz.~\cite{ChDiMa14,BPdB15,EGSS13c}) effectively reduce the problem
to propositional model counting.  Such approaches are similar 
to ``bit blasting'' in SMT
solvers~\cite{KSBook08}.

The primary contribution of this paper is an efficient word-level
approximate model counting algorithm {\SMTApproxMC} that can be employed to
answer inference queries over high-dimensional discrete
domains.  Our algorithm uses a new class of word-level hash
functions that are $2$-universal and can be solved by word-level \SMT
 solvers capable of reasoning about linear equalities on
words. Therefore, unlike previous works, {\SMTApproxMC} is able to
leverage the power of sophisticated {\SMT} solvers.

To illustrate the practical utility of {\SMTApproxMC}, we implemented
a prototype and evaluated it on a suite of benchmarks.  Our
experiments demonstrate that {\SMTApproxMC} can significantly
outperform the prevalent approach of bit-blasting a word-level
constraint and using an approximate propositional model counter that
employs XOR-based hash functions.  
Our proposed word-level hash functions embed the domain of all variables
in a large enough finite domain. Thus, one would not expect our approach
to work well for constraints that exhibit a hugely heterogeneous mix of word widths,
or for problems that are difficult for word-level {\SMT} solvers.
Indeed, our experiments suggest that the use of word-level
hash functions 
provides significant benefits when the original word-level constraint is such
that (i)~the words appearing in it have long and similar widths, and
(ii)~the \SMT solver can reason about the constraint at the
word-level, without extensive bit-blasting.  
\section{Preliminaries \label{sec:prelim}}
A \emph{word} (or \emph{bit-vector}) is an array of bits. The size of
the array is called the \emph{width} of the word. We consider here
\emph{fixed-width} words, whose width is a constant.  It is easy to see 
that a word of width $k$ can be used to represent elements of a set of 
size $2^k$.  The first-order theory of fixed-width words has been extensively
studied (see~\cite{KSBook08,BruttomessoThesis08} for an overview).
The vocabulary of this theory includes interpreted predicates and
functions, whose semantics are defined over words interpreted as
signed integers, unsigned integers, or vectors of propositional
constants (depending on the function or predicate).  When a word of
width $k$ is treated as a vector, we assume that the component bits
are indexed from $0$ through $k-1$, where index $0$ corresponds to the
rightmost bit.  %
A \emph{term}
is either a word-level variable or constant, or is obtained by
applying functions in the vocabulary to a term.  Every term has an
associated width that is uniquely defined by the widths of word-level
variables and constants in the term, and by the semantics of functions
used to build the term.  For purposes of this paper, given terms $t_1$
and $t_2$, we use $t_1 + t_2$ (resp. $t_1 * t_2)$ to
denote the sum (resp. product) of $t_1$ and $t_2$, interpreted as
unsigned integers.  Given a positive integer $p$, we use $t_1 \mod p$
to denote the remainder after dividing $t_1$ by $p$.
Furthermore, if $t_1$ has width $k$, and $a$ and $b$ are integers
such that $0 \le a \le b < k$, we use $\extract{t_1}{a}{b}$ to denote
the slice of $t_1$ (interpreted as a vector) between indices $a$ and
$b$, inclusively.

Let $F$ be a formula in the theory of fixed-width words.  The
\emph{support} of $F$, denoted $\mathsf{sup}(F)$, is the set of
word-level variables that appear in $F$.  A \emph{model} or
\emph{solution} of $F$ is an assignment of word-level constants to
variables in $\mathsf{sup}(F)$ such that $F$ evaluates to
$\mathsf{True}$.  We use $\satisfying{F}$ to denote the set of
\emph{models} of $F$.
The model-counting problem requires us to compute $|\satisfying{F}|$.  
For simplicity of exposition, we assume henceforth
that all words in $\mathsf{sup}(F)$ have the same width.  Note that
this is without loss of generality, since if $k$ is the maximum
width of all words in $\mathsf{sup}(F)$, we can construct a formula
$\widehat{F}$ such that the following hold: (i)~$|\mathsf{sup}(F)| =
|\mathsf{sup}(\widehat{F})|$, (ii)~all word-level variables in
$\widehat{F}$ have width $k$, and (iii)~$|R_F| = |R_{\widehat{F}}|$.
The formula $\widehat{F}$ is obtained by replacing every occurrence of
word-level variable $x$ having width $m ~(< k)$ in $F$ with
$\extract{\widehat{x}}{0}{m-1}$, where $\widehat{x}$ is a new variable
of width $k$.

We write $\prob\left[X: {\cal P} \right]$ for the probability of
outcome $X$ when sampling from a probability space ${\cal P}$.  For
brevity, we omit ${\cal P}$ when it is clear from the context.  

Given a word-level formula $F$, an \emph{exact model counter} returns
$|\satisfying{F}|$.  An \emph{approximate model counter} relaxes this
requirement to some extent: given a \emph{tolerance} $\varepsilon > 0$
and \emph{confidence} $1-\delta \in (0, 1]$, the value $v$ returned by
the counter satisfies
$\prob[\frac{|{\satisfying{F}}|}{1+\varepsilon} \le v \le
    (1+\varepsilon)|{\satisfying{F}}|] \ge 1-\delta$.
Our model-counting algorithm belongs to the class of approximate model
counters.

Special classes of hash functions, called \emph{$2$-wise independent
  universal} hash functions play a crucial role in our work.  Let
$\mathsf{sup}(F) = \{x_0, \ldots x_{n-1}\}$, where each $x_i$ is a word of
width $k$.  The space of all assignments of words in $\mathsf{sup}(F)$
is $\{0, 1\}^{n.k}$.  We use hash functions that map elements of $\{0,
1\}^{n.k}$ to $p$ bins labeled $0, 1, \ldots p-1$, where $1 \le p <
2^{n.k}$.  Let $\Z_p$ denote $\{0, 1, \ldots p-1\}$ and let ${\HH}$
denote a family of hash functions mapping $\{0, 1\}^{n.k}$ to $\Z_p$.
We use $h \xleftarrow{R} {\HH}$ to denote the probability space obtained
by choosing a hash function $h$ uniformly at random from ${\HH}$.  We say
that ${\HH}$ is a $2$-wise independent universal hash family if for all
$\alpha_1, \alpha_2 \in \Z_p$ and for all distinct $\X_1, \X_2 \in
\{0,1\}^{n.k}$, $\prob\left[h(\X_1) = \alpha_1 \wedge h(\X_2) =
  \alpha_2: h \xleftarrow{R} {\HH} \right] = 1/p^2$.

\section{Related Work}\label{sec:relwork}
The connection between probabilistic inference and model counting has
been extensively studied by several
authors~\cite{Cooper90,Roth1996,chavira2008probabilistic}, and it is
known that the two problems are inter-reducible.  Propositional model
counting was shown to be \#P-complete by
Valiant~\cite{Valiant1979complexity}.  It follows easily that the
model counting problem for fixed-width words is also \#P-complete.  It
is therefore unlikely that efficient exact algorithms exist for this
problem. 
\cite{Bellare00} showed that a closely related
problem, that of almost uniform sampling from propositional
constraints, can be solved in probabilistic polynomial time using an
{\NP} oracle.  
Subsequently, ~\cite{Jerr} showed that approximate model
counting is polynomially inter-reducible to almost uniform sampling.
While this shows that approximate model counting is solvable in
probabilstic polynomial time relative to an {\NP} oracle, the
algorithms resulting from this largely theoretical body of work are highly
inefficient in practice~\cite{Meel14}.

Building on the work of Bellare, Goldreich and Petrank~\shortcite{Bellare00}, Chakraborty, Meel and Vardi~\shortcite{CMV13b} proposed the
first scalable approximate model counting algorithm for propositional
formulas, called {\ApproxMC}.  Their technique is based on the use of
a family of $2$-universal bit-level hash functions that compute XOR of
randomly chosen propositional variables.  Similar bit-level hashing
techniques were also used in~\cite{EGSS13c,CFMSV14} for weighted model
counting. All of these works leverage the significant advances made
in propositional satisfiability solving in the recent past~\cite{BHMW09}.%

Over the last two decades, there has been tremendous progress in the
development of decision procedures, called Satisfiability Modulo
Theories (or {\SMT}) solvers, for combinations of first-order
theories, including the theory of fixed-width
words~\cite{BarrettFT10,BMS05}.  An {\SMT} solver uses a core propositional
reasoning engine and decision procedures for individual theories, to
determine the satisfiability of a formula in the combination of
theories.  It is now folklore that a well-engineered word-level {\SMT}
solver can significantly outperform the naive approach of
\emph{blasting} words into component bits and then using a
propositional satisfiability solver~\cite{DB08,JLS09,BCFGHNPS07}.  The
power of word-level {\SMT} solvers stems from their ability to reason
about words directly (e.g. $a + (b - c) = (a - c) + b$ for every word
$a$, $b$, $c$), instead of \emph{blasting} words into component bits
and using propositional reasoning.

The work of~\cite{ChDiMa14} tried to extend
{\ApproxMC}~\cite{CMV13b} to
non-propositional domains.  A crucial step in their approach is to
propositionalize the solution space (e.g. bounded integers are equated
to tuples of propositions) and then use XOR-based bit-level hash
functions.  Unfortunately, such propositionalization can significantly
reduce the effectiveness of theory-specific reasoning in an {\SMT}
solver.  The work of~\cite{BPdB15} used bit-level hash functions with
the propositional abstraction of an {\SMT} formula to solve the
problem of \emph{weighted model integration}. This approach also fails
to harness the power of theory-specific reasoning in {\SMT} solvers.

Recently,~\cite{BRGD15} proposed $\mathsf{SGDPLL}(T)$, an algorithm
that generalizes {\SMT} solving to do lifted inferencing and model
counting (among other things) modulo background theories (denoted
$T$).  %
A fixed-width word model counter, like the one proposed in this paper,
can serve as a theory-specific solver in the $\mathsf{SGDPLL}(T)$
framework.  In addition, it can also serve as an alernative to
$\mathsf{SGDPLL}(T)$ when the overall problem is simply to count models in
the theory $T$ of fixed-width words,
There have also been other attempts to exploit the power of {\SMT}
solvers in machine learning.  For example,~\cite{TSP14} used
optimizing {\SMT} solvers for structured relational learning using
Support Vector Machines.  This is unrelated to our approach
of harnessing the power of {\SMT} solvers for probabilistic inference
via model counting.

\section{Word-level Hash Function} \label{sec:hash}
The performance of hashing-based techniques for approximate model
counting depends crucially on the underlying family of hash functions
used to partition the solution space.  A popular family of hash
functions used in propositional model counting is ${\HH}_{xor}$,
defined as the family of functions obtained by XOR-ing a random subset
of propositional variables, and equating the result to either $0$ or
$1$, chosen randomly.  The family ${\HH}_{xor}$ enjoys important
properties like $2$-independence and easy implementability, which make
it ideal for use in practical
model counters for propositional
formulas~\cite{Gomes-Sampling,EGSS13c,CMV13b}.  Unfortunately, word-level
universal hash families that are $2$-independent, easily implementable
and amenable to word-level reasoning by {\SMT} solvers, have not been
studied thus far.  %
In this section, we present ${\HH}_{SMT}$, a family of word-level hash
functions that fills this gap.%
As discussed earlier, let $\mathsf{sup}(F) = \{x_0, \ldots x_{n-1}\}$,
where each $x_i$ is a word of width $k$.  We use $\X$ to denote the
$n$-dimensional vector $(x_0, \ldots x_{n-1})$.  The space of all
assignments to words in $\X$ is $\{0,1\}^{n.k}$.  Let $p$ be a prime
number such that $2^k \leq p < 2^{n.k}$.
Consider a family ${\HH}$ of hash functions mapping $\{0,1\}^{n.k}$ to
$\Z_p$, where each hash function is of the form $h(\X) =
(\sum_{j=0}^{n-1}a_{j}*x_j + b)\mod p$, and the $a_{j}$'s and $b$ are
elements of $\Z_p$, represented as words of width $\ceil{\log_2 p}$.
Observe that every $h \in {\HH}$ partitions $\{0,1\}^{n.k}$ into $p$
bins (or cells).  Moreover, for every $\xi \in \{0,1\}^{n.k}$ and
$\alpha \in \Z_p$, $\prob\left[h(\xi) = \alpha : h \xleftarrow{R}
  {\HH}\right] = p^{-1}$.  For a hash function chosen uniformly at
random from ${\HH}$, the expected number of elements per cell is
$2^{n.k}/p$.  Since $p < 2^{n.k}$, every cell has at least $1$ element
in expectation.  Since $2^k \leq p$, for every word
$x_i$ of width $k$, we also have $x_i \mod p = x_i$. Thus, distinct
words are not aliased (or made to behave similarly) because of modular
arithmetic in the hash function.

Suppose now we wish to partition $\{0.1\}^{n.k}$ into $p^c$ cells,
where $c > 1$ and $p^c < 2^{n.k}$. To achieve this, we need to define
hash functions that map elements in $\{0,1\}^{n.k}$ to a tuple in
$\left(\Z_p\right)^c$.  A simple way to achieve this is to take
a $c$-tuple of hash functions, each of which maps $\{0,1\}^{n.k}$ 
to $\Z_p$.
Therefore, the desired family of hash functions is simply the iterated
Cartesian product ${\HH} \times \cdots \times {\HH}$, where the
product is taken $c$ times.  Note that every hash function in this
family is a $c$-tuple of hash functions.  For a hash function chosen
uniformly at random from this family, the expected number of elements
per cell is $2^{n.k}/p^c$.

An important consideration in hashing-based techniques for approximate
model counting is the choice of a hash function that yields cells that
are neither too large nor too small in their expected sizes. Since
increasing $c$ by $1$ reduces the expected size of each cell by a
factor of $p$, it may be difficult to satisfy the above requirement if
the value of $p$ is large.  At the same time, it is desirable to have
$p > 2^k$ to prevent aliasing of two distinct words of width $k$.
This motivates us to consider more general classes of word-level hash
functions, in which each word $x_i$ can be split into thinner slices,
effectively reducing the width $k$ of words, and allowing us to use
smaller values of $p$.  We describe this in more detail below.

Assume for the sake of simplicity that $k$ is a power of $2$, and let
$q$ be $\log_2 k$.  For every $j \in \{0, \ldots q-1\}$ and for every
$x_i \in \X$, define $\Slice{\mathbf{x_i}}{j}$ to be the
$2^j$-dimensional vector of slices of the word $x_i$, where each slice
is of width $k/2^j$.  For example, the two slices in
$\Slice{\mathbf{x_1}}{1}$ are $\extract{x_1}{0}{k/2-1}$ and
$\extract{x_1}{k/2}{k-1}$.  Let $\X^{(j)}$ denote the
$n.2^j$-dimensional vector $(\mathbf{x_0}^{(j)}, \mathbf{x_1}^{(j)},
\ldots \mathbf{x_{n-1}}^{(j)})$.  It is easy to see that the $m^{th}$
component of $\X^{(j)}$, denoted $\X_{m}^{(j)}$, is
$\extract{x_i}{s}{t}$, where $i = \floor{m/2^{j}}$, $s = (m \mod
2^j)\cdot(k/2^j)$ and $t = s + (k/2^j) -1$.  Let $p_j$ denote the
smallest prime larger than or equal to $2^{(k/2^j)}$.  Note that this
implies $p_{j+1} \leq p_j$ for all $j \ge 0$.  In order to obtain a
family of hash functions that maps $\{0,1\}^{n.k}$ to $\Z_{p_j}$, we
split each word $x_i$ into slices of width $k/2^j$, treat these slices
as words of reduced width, and use a technique similar to the one used
above to map $\{0,1\}^{n.k}$ to $\Z_{p}$.  Specifically, the family
${\HH}^{(j)} = \left\{h^{(j)} : h^{(j)}(\X) =
\left(\sum_{m=0}^{n.2^j-1} a_{m}^{(j)}*\X_{m}^{(j)} + b^{(j)}\right)
\mod p_j\right\}$ maps $\{0,1\}^{n.k}$ to $\Z_{p_j}$, where the values
of $a_{m}^{(j)}$ and $b^{(j)}$ are chosen from $\Z_{p_j}$, and
represented as $\ceil{\log_2 p_j}$-bit words.

  In general, we may wish to define a family of hash functions that
  maps $\{0,1\}^{n.k}$ to $\mc{D}$, where $\mc{D}$ is given by
  $\left(\Z_{p_0}\right)^{c_0} \times \left(\Z_{p_1}\right)^{c_1}
  \times \cdots \left(\Z_{p_{q-1}}\right)^{c_{q-1}}$ and
  $\prod_{j=0}^{q-1}p_j^{c_j} < 2^{n.k}$.  To achieve this, we first
  consider the iterated Cartesian product of ${\HH}^{(j)}$ with itself
  $c_j$ times, and denote it by $\left({\HH}^{(j)}\right)^{c_j}$, for
  every $j \in \{0, \ldots q-1\}$.  Finally, the desired family of
  hash functions is obtained as $\prod_{j=0}^{q-1}
  \left({\HH}^{(j)}\right)^{c_j}$.  Observe that every hash function
  $h$ in this family is a $\left(\sum_{l=0}^{q-1} c_l\right)$-tuple of
  hash functions.  Specifically, the $r^{th}$ component of $h$, for $r
  \le \left(\sum_{l=0}^{q-1} c_l\right)$, is given by
  $\left(\sum_{m=0}^{n.2^j-1} a_{m}^{(j)}* \X_{m}^{(j)} +
  b^{(j)}\right)\mod p_{j}$, where $\left(\sum_{i=0}^{j-1} c_i\right)
  < r \le \left(\sum_{i=0}^{j} c_i\right)$, and the $a_{m}^{(j)}$s and
  $b^{(j)}$ are elements of $\Z_{p_j}$.%

  The case when $k$ is not a power of $2$ is handled by splitting the
  words $x_i$ into slices of size $\ceil{k/2}$, $\ceil{k/2^2}$ and so
  on.  Note that the family of hash functions defined above depends
  only on $n$, $k$ and the vector $C = (c_0, c_1, \ldots c_{q-1})$,
  where $q = \ceil{\log_2 k}$.  Hence, we call this family
  ${\HH}_{SMT}(n, k, C)$.
Note also that by setting $c_i$ to $0 $ for all $i \neq \lfloor \log_2(k/2)
\rfloor$, and $c_i$ to $r $ for $i = \lfloor \log_2(k/2) \rfloor$ reduces
${\HH}_{SMT}$ to the family ${\HH}_{xor}$ of XOR-based bit-wise hash
functions mapping $\{0,1\}^{n.k}$ to $\{0,1\}^r$.  Therefore,
$H_{SMT}$ strictly generalizes ${\HH}_{xor}$.  

We summarize below important properties of the ${\HH}_{SMT}(n,k,C)$
class.  All proofs are
available in full version.. %
\begin{lemma} \label{lm:universal}
For every $\X \in \{0,1\}^{n.k}$ and every $\alpha \in \mc{D}$, $\prob[h(\X) = \alpha \mid h \xleftarrow{R} {\HH}_{SMT}(n,k,C)] = \prod_{j=0}^{|C|-1}{p_j}^{-c_j}$
\end{lemma}
\begin{theorem}\label{th-universal}
For every $\alpha_1, \alpha_2 \in \mc{D}$ and every distinct 
$\X_1,\X_2 \in \{0,1\}^{n.k}$, $\prob[(h(\X_1) = \alpha_1 \wedge h(\X_2) = \alpha_2 )  \mid 
h \xleftarrow{R}  {\HH}_{SMT}(n,k,C)] = \prod_{j=0}^{|C|-1}({p_j})^{-2.c_j}$. 
Therefore, $ {\HH}_{SMT}(n,k,C)$ is pairwise independent.
\end{theorem}
\paragraph{Gaussian Elimination}
The practical success of XOR-based bit-level hashing techniques for
propositional model counting owes a lot to solvers like
CryptoMiniSAT~\cite{SNC09} that use Gaussian Elimination to
efficiently reason about XOR constraints.  It is significant that the
constraints arising from ${\HH}_{SMT}$ are linear modular equalities
that also lend themselves to efficient Gaussian Elimination.  %
We believe that integration of Gaussian Elimination engines in {\SMT}
solvers will significantly improve the performance of
hashing-based word-level model counters.
\section{Algorithm}\label{sec:algo}
We now present {\SMTApproxMC}, a word-level hashing-based approximate
model counting algorithm.  {\SMTApproxMC} takes as inputs a formula
$F$ in the theory of fixed-width words, a tolerance $\varepsilon ~(>
0)$, and a confidence $1-\delta \in (0,1]$.  It returns an estimate of
  $|\satisfying{F}|$ within the tolerance $\varepsilon$, with
  confidence $1-\delta$.  The formula $F$ is assumed to have $n$
  variables, each of width $k$, in its support. The central idea of
  {\SMTApproxMC} is to randomly partition the solution space of $F$
  into ``small" cells of roughly the same size, using word-level hash
  functions from ${\HH}_{SMT}(n, k, C)$, where $C$ is incrementally
  computed.  The check for ``small''-ness of cells is done using a
  word-level {\SMT} solver.  The use of word-level hash functions and
  a word-level {\SMT} solver allows us to directly harness the power
  of {\SMT} solving in model counting.

  The pseudocode for {\SMTApproxMC} is presented in
  Algorithm~\ref{alg:SMTMC}.
  Lines~\ref{line:weightmc-init-start}--~\ref{line:weightmc-init-end}
  initialize the different parameters. Specifically, $\mathrm{pivot}$
  determines the maximum size of a ``small'' cell as a function of
  $\varepsilon$, and $t$ determines the number of times
  {\SMTApproxMCCore} must be invoked, as a function of $\delta$. The value of $t$ is determined by technical arguments in the 
  proofs of our theoretical guarantees, and is not based on experimental 
  observations %
  Algorithm {\SMTApproxMCCore} lies at the heart of {\SMTApproxMC}.
  Each invocation of {\SMTApproxMCCore} either returns an approximate
  model count of $F$, or $\bot$ (indicating a failure).  In the former
  case, we collect the returned value, $m$, in a list $M$ in
  line~\ref{line:weightmc-loop-core-update}. Finally, we compute the
  median of the approximate counts in $M$, and return this as
  $\mathrm{FinalCount}$.
\begin{algorithm}[h]
	\caption{\SMTApproxMC$(F, \varepsilon, \delta,k)$}
	\label{alg:SMTMC}
	\begin{algorithmic}[1]
		
		\State $\mathrm{counter} \gets 0; M \gets \mathsf{emptyList}; $ \label{line:weightmc-init-start}
		\State $\mathrm{pivot} \gets 2 \times \lceil e^{-3/2} \left(1 + \frac{1}{\varepsilon}\right)^2 \rceil$;
		\State $t \gets \left\lceil 35\log_2 (3/\delta) \right\rceil$; \label{line:weightmc-init-end}
		\Repeat \label{line:weightmc-loop-start}
		\State $m \gets \SMTApproxMCCore(F,\mathrm{pivot},k)$; \label{line:weightmc-loop-core-invokation}
		\State $\mathrm{counter} \gets \mathrm{counter}+1$;
		\If {$ m \neq \bot$}
		\State $ \mathsf{AddToList}(M,m)$;
		\label{line:weightmc-loop-core-update}
		\EndIf
		\Until { $(\mathrm{counter} < t)$} \label{line:weightmc-loop-end}
		\State $\mathrm{FinalCount} \gets \mathsf{FindMedian}(M)$; \label{line:weightmc-median}
		\State \Return $\mathrm{FinalCount}$;
		\label{line:weightmc-return}
	\end{algorithmic}
\end{algorithm}

The pseudocode for {\SMTApproxMCCore} is shown in
Algorithm~\ref{alg:SMTMCCore}.  This algorithm takes as inputs a
word-level {\SMT} formula $F$, a threshold $\mathrm{pivot}$, and the
width $k$ of words in $\mathsf{sup}(F)$.  We assume access to a
subroutine {\BoundedSMT} that accepts a word-level {\SMT} formula
$\varphi$ and a threshold $\mathrm{pivot}$ as inputs, and returns
$\mathrm{pivot}+1$ solutions of $\varphi$ if $|\satisfying{\varphi}| >
\mathrm{pivot}$; otherwise it returns $\satisfying{\varphi}$. In
lines~\ref{line:SMTMCCore-small-check-init}--~\ref{line:SMTMCCore-small-check-end}
of Algorithm~\ref{alg:SMTMCCore}, we return the exact count if $|R_F|
\leq \mathrm{pivot}$.  Otherwise, we initialize $C$ by setting $C[0]$
to $0$ and $C[1]$ to $1$, where $C[i]$ in the pseudocode refers to
$c_i$ in the previous section's discussion.  This choice of
initialization is motivated by our experimental observations.  We also
count the number of cells generated by an arbitrary hash function from
${\HH}_{SMT}(n,k,C)$ in $\mathrm{numCells}$.  The loop in
lines~\ref{line:SMTMCCore-loop-start}--\ref{line:SMTMCCore-loop-end}
iteratively partitions $R_F$ into cells using randomly chosen hash
functions from ${\HH}_{SMT}(n, k, C)$.  The value of $i$ in each
iteration indicates the extent to which words in the support of $F$
are sliced when defining hash functions in ${\HH}_{SMT}(n, k, C)$ --
specifically, slices that are $\ceil{k/2^i}$-bits or more wide are
used.  The iterative partitioning of $R_F$ continues until a randomly
chosen cell is found to be ``small'' (i.e. has $\geq 1$ and $\leq
\mathrm{pivot}$ solutions), or the number of cells exceeds $2^{n.k}$,
rendering further partitioning meaningless.  The random choice of $h$
and $\alpha$ in lines~\ref{line:SMTMCCore-choose-hash}
and~\ref{line:SMTMCCore-choose-alpha} ensures that we pick a random
cell.  The call to {\BoundedSMT} returns at most $\mathrm{pivot}+1$
solutions of $F$ within the chosen cell in the set $Y$.
If $|Y| > \mathrm{pivot}$, the cell is deemed to be large, and the
algorithm partitions each cell further into $p_i$ parts.  This is done
by incrementing $C[i]$ in line~\ref{line:SMTMCCore-increase-c-i}, so
that the hash function chosen from ${\HH}_{SMT}(n,k,C)$ in the next
iteration of the loop generates $p_i$ times more cells than in the
current iteration.  On the other hand, if $Y$ is empty and $p_i > 2$,
the cells are too small (and too many), and the algorithm reduces the
number of cells by a factor of $p_{i+1}/p_{i}$ (recall $p_{i+1} \leq
p_i$) by setting the values of $C[i]$ and $C[i+1]$ accordingly (see
lines\ref{line:SMTMCCore-refine-c-i-start}
--\ref{line:SMTMCCore-refine-c-i-end}).  If $Y$ is non-empty and has
no more than $\mathrm{pivot}$ solutions, the cells are of the right
size, and we return the estimate $|Y|\times \mathrm{numCells}$.  In
all other cases, ${\SMTApproxMCCore}$ fails and returns $\bot$.

\begin{algorithm}
	
	\caption{\SMTApproxMCCore$(F,\mathrm{pivot},k)$}
	\label{alg:SMTMCCore}
	\begin{algorithmic}[1]
		\State $Y \gets \BoundedSMT(F,\mathrm{pivot})$;
		\label{line:SMTMCCore-small-check-init}
		\If {$ |Y| \leq \mathrm{pivot})$} \Return $|Y|$;
		\label{line:SMTMCCore-small-check-end}
		\Else
			\State $C \gets \text{emptyVector}$; $C[0] \gets 0$; $C[1] \gets 1$;
			\State $i \gets 1$;~ $\text{numCells} \gets p_1$;
		
				\Repeat \label{line:SMTMCCore-loop-start}
					\State Choose $h$ at random from ${\HH}_{SMT}(n,k,C)$; \label{line:SMTMCCore-choose-hash}
					\State Choose $\alpha$ at random from $\prod_{j = 0}^i \left(\Z_{p_j}\right)^{C[j]}$; \label{line:SMTMCCore-choose-alpha}
					\State $Y \gets \BoundedSMT(F \wedge (h(\X) = \alpha), \mathrm{pivot})$;
					 \label{line:SMTMCCore-call-BoundedWeightSAT}
					 \If {($|Y| > \mathrm{pivot}$)}\label{line:SMTMCCore-cell-comparison}
					 	\State $C[i] \gets C[i]+1$;\label{line:SMTMCCore-increase-c-i}
                                                \State $\text{numCells} \gets \text{numCells}\times p_i$;
					 \EndIf	
					  \If {($|Y| = 0$)}
					  	\If {$p_i > 2$}
					  		\State $C[i] \gets C[i]-1$;
					  		\label{line:SMTMCCore-refine-c-i-start}
					  		\State $i \gets i+1$; $C[i] \gets 1$; 
                                                        \State $\text{numCells} \gets \text{numCells}\times(p_{i+1}/p_i)$;
					  			\label{line:SMTMCCore-refine-c-i-end}
					  	\Else
					  		\State break; \label{line:SMTMCCore-break}
					  	\EndIf
					  \EndIf
				\Until {$((0 < |Y| \leq \mathrm{pivot})$ or $(\text{numCells} > 2^{n.k}))$}  \label{line:SMTMCCore-loop-end}
			\If {$((|Y| > \mathrm{pivot})$ {\bfseries or} $(|Y| = 0))$}
				\Return $\bot$;
			\Else 
				~~\Return $ |Y|\times \text{numCells}$; \label{line:SMTMCCore-return}
			\EndIf
		\EndIf
	\end{algorithmic}
\end{algorithm}

Similar to the analysis of {\ApproxMC}~\cite{CMV13b}, the current theoretical analysis of {\SMTApproxMC} assumes that for some $C$ during the execution of ${\SMTApproxMCCore}$, $\log |R_F| - \log(\text{numCells}) +1 = \log(\mathrm{pivot})$. We leave analysis of {\SMTApproxMC} without above assumption to future work. The following theorems concern the correctness and performance of
{\SMTApproxMC}.  %
\begin{theorem}\label{thm:guarantees}
Suppose an invocation of {\SMTApproxMC}$(F, \varepsilon, \delta, k)$
returns $\mathrm{FinalCount}$.  Then $\prob\left[
  (1+\varepsilon)^{-1}|R_F| \leq \mathrm{FinalCount} \leq
  (1+\varepsilon)|R_F| \right] \geq 1-\delta$
\end{theorem}
\begin{theorem}\label{thm:complexity}
{\SMTApproxMC}$(F, \varepsilon, \delta, k)$ runs in time polynomial in $|F|$,
$1/\varepsilon$ and $\log_2(1/\delta)$ relative to an {\NP}-oracle.%
\end{theorem}

The proofs of Theorem~\ref{thm:guarantees} and ~\ref{thm:complexity} can be found in full version..
\section{Experimental Methodology and Results}\label{sec:expts}
\begin{table*}
	\scriptsize
	\centering
	\begin{tabular}{|c|c|c|c|c|c|}
		\hline
		Benchmark & Total Bits & Variable Types & \# of Operations & \shortstack{SMTApproxMC\\ time(s)} & \shortstack{CDM\\ time(s)} \\
		\hline
		squaring27 & 59 & \{1: 11, 16: 3\} & 10 & -- & 2998.97 \\ \hline
		squaring51 & 40 & \{1: 32, 4: 2\} & 7 & 3285.52 & 607.22 \\ \hline
		1160877 & 32 & \{8: 2, 16: 1\} & 8 & 2.57 & 44.01 \\ \hline
		1160530 & 32 & \{8: 2, 16: 1\} & 12 & 2.01 & 43.28 \\ \hline
		1159005 & 64 & \{8: 4, 32: 1\} & 213 & 28.88 & 105.6 \\ \hline
		1160300 & 64 & \{8: 4, 32: 1\} & 1183 & 44.02 & 71.16 \\ \hline
		1159391 & 64 & \{8: 4, 32: 1\} & 681 & 57.03 & 91.62 \\ \hline
		1159520 & 64 & \{8: 4, 32: 1\} & 1388 & 114.53 & 155.09 \\ \hline
		1159708 & 64 & \{8: 4, 32: 1\} & 12 & 14793.93 & -- \\ \hline
		1159472 & 64 & \{8: 4, 32: 1\} & 8 & 16308.82 & -- \\ \hline
		1159115 & 64 & \{8: 4, 32: 1\} & 12 & 23984.55 & -- \\ \hline
		1159431 & 64 & \{8: 4, 32: 1\} & 12 & 36406.4 & -- \\ \hline
		1160191 & 64 & \{8: 4, 32: 1\} & 12 & 40166.1 & -- \\ \hline
	\end{tabular}
	\scriptsize
	\caption{Runtime performance of {\SMTApproxMC} vis-a-vis {\CDM} for a subset of benchmarks.}
	\label{tab:performance-comparison}
\end{table*}
To evaluate the performance and effectiveness of {\SMTApproxMC}, we
built a prototype implementation and conducted extensive
experiments. Our suite of benchmarks consisted of more than $150$
problems arising from diverse domains such as reasoning about
circuits, planning, program synthesis and the like. For lack
of space, we present results for only for a subset of the benchmarks.

For purposes of comparison, we also implemented a state-of-the-art
bit-level hashing-based approximate model counting algorithm for
bounded integers, proposed by~\cite{ChDiMa14}.  Henceforth, we refer
to this algorithm as {\CDM}, after the authors' initials. Both model
counters used an overall timeout of $12$ hours per benchmark, and a
{\BoundedSMT} timeout of $2400$ seconds per call.  Both used
{\Boolector}, a state-of-the-art {\SMT} solver for fixed-width
words~\cite{BB09}.  Note that {\Boolector} (and other popular {\SMT}
solvers for fixed-width words) does not yet implement Gaussian
elimination for linear modular equalities; hence our experiments did
not enjoy the benefits of Gaussian elimination. We employed the
Mersenne Twister to generate pseudo-random numbers, and each thread
was seeded independently using the Python \texttt{random} library.
All experiments used $\varepsilon = 0.8$ and $\delta = 0.2$. Similar to {\ApproxMC}, we determined value of $t$ based on tighter analysis offered by proofs. For detailed discussion, we refer the reader to Section 6 in ~\cite{CMV13b}. Every
experiment was conducted on a single core of high-performance computer
cluster, where each node had a 20-core, 2.20 GHz Intel Xeon processor,
with 3.2GB of main memory per core.

We sought answers to the following questions from our experimental
evaluation:
\begin{enumerate}
\item How does the performance of {\SMTApproxMC} compare with that
  of a bit-level hashing-based counter like {\CDM}?
\item How do the approximate counts returned by {\SMTApproxMC} 
compare with exact counts?
\end{enumerate}
Our experiments show that {\SMTApproxMC} significantly outperforms
{\CDM}  for a large class of benchmarks.
Furthermore, the counts returned by {\SMTApproxMC} are highly accurate
and the observed geometric tolerance($\varepsilon_{obs}$) = $0.04$.

\paragraph{Performance Comparison}
Table~\ref{tab:performance-comparison} presents the result of
comparing the performance of {\SMTApproxMC} vis-a-vis {\CDM} on a subset of
our benchmarks. In Table~\ref{tab:performance-comparison}, column
1 gives the benchmark identifier, column 2 gives the sum of widths of
all variables, column 3 lists the number of variables ($\mathrm{numVars}$) for
each corresponding width ($\mathrm{w}$) in the format $\{\mathrm{w}:\mathrm{numVars}\}$. To
indicate the complexity of the input formula, we present the number of
operations in the original {\SMT} formula in column 4. The runtimes
for {\SMTApproxMC} and {\CDM} are presented in columns 5 and column 6
respectively.  We use ``--" to denote timeout after $12$ hours.
Table~\ref{tab:performance-comparison} clearly shows that
{\SMTApproxMC} significantly outperforms {\CDM} (often by $2$-$10$
times) for a large class of benchmarks. In particular, we observe that
{\SMTApproxMC} is able to compute counts for several cases where
{\CDM} times out. %

Benchmarks in our suite exhibit significant heterogeneity in the
widths of words, and also in the kinds of word-level operations used.
Propositionalizing all word-level variables eagerly, as is done in
{\CDM}, prevents the {\SMT} solver from making full use of word-level
reasoning.  In contrast, our approach allows the power of word-level
reasoning to be harnessed if the original formula $F$ and the hash
functions are such that the {\SMT} solver can reason about them
without bit-blasting.  This can lead to significant performance
improvements, as seen in Table~\ref{tab:performance-comparison}.
Some benchmarks, however, have heterogenous bit-widths and
heavy usage of operators like $\extract{x}{n_1}{n_2}$ and/or
word-level multiplication.  It is known that word-level reasoning in
modern {\SMT} solvers is not very effective for such cases, and the
solver has to resort to bit-blasting.  Therefore, using word-level
hash functions does not help in such cases. %
We believe this contributes to
the degraded performance of {\SMTApproxMC} vis-a-vis {\CDM} in a
subset of our benchmarks.  This also points to an interesting direction of
future research: to find the right hash function for a benchmark by utilizing {\SMT} solver's architecture.

\paragraph*{Quality of Approximation}
To measure the quality of the counts returned by {\SMTApproxMC}, we
selected a subset of benchmarks that were small enough to be
bit-blasted and fed to {\sharpSATTool}~\cite{Thurley2006} -- a
state-of-the-art exact model counter.  Figure~\ref{fig:quality}
compares the model counts computed by {\SMTApproxMC} with the bounds
obtained by scaling the exact counts (from {\sharpSATTool}) with
the tolerance factor $(\varepsilon = 0.8)$. The y-axis represents
model counts on log-scale while the x-axis presents benchmarks ordered
in ascending order of model counts. We observe that for {\em all} the
benchmarks, {\SMTApproxMC} computes counts within the tolerance.
Furthermore, for each instance, we computed observed tolerance (
$\varepsilon_{obs}$) as $\frac{\mathrm{count}}{|R_F|} -1$, if $\mathrm{count}
\geq |R_F|$, and $\frac{|R_F|}{\mathrm{count}}-1$ otherwise, where
$|R_F|$ is computed by {\sharpSATTool} and $\mathrm{count}$ is computed
by {\SMTApproxMC}. We observe that the geometric mean of $\varepsilon_{obs}$ across all the benchmarks is only $0.04$ -- far less (i.e. closer to the exact count) than the theoretical guarantee of
$0.8$.  

\begin{figure}

\centering
\includegraphics[width=0.5\textwidth]{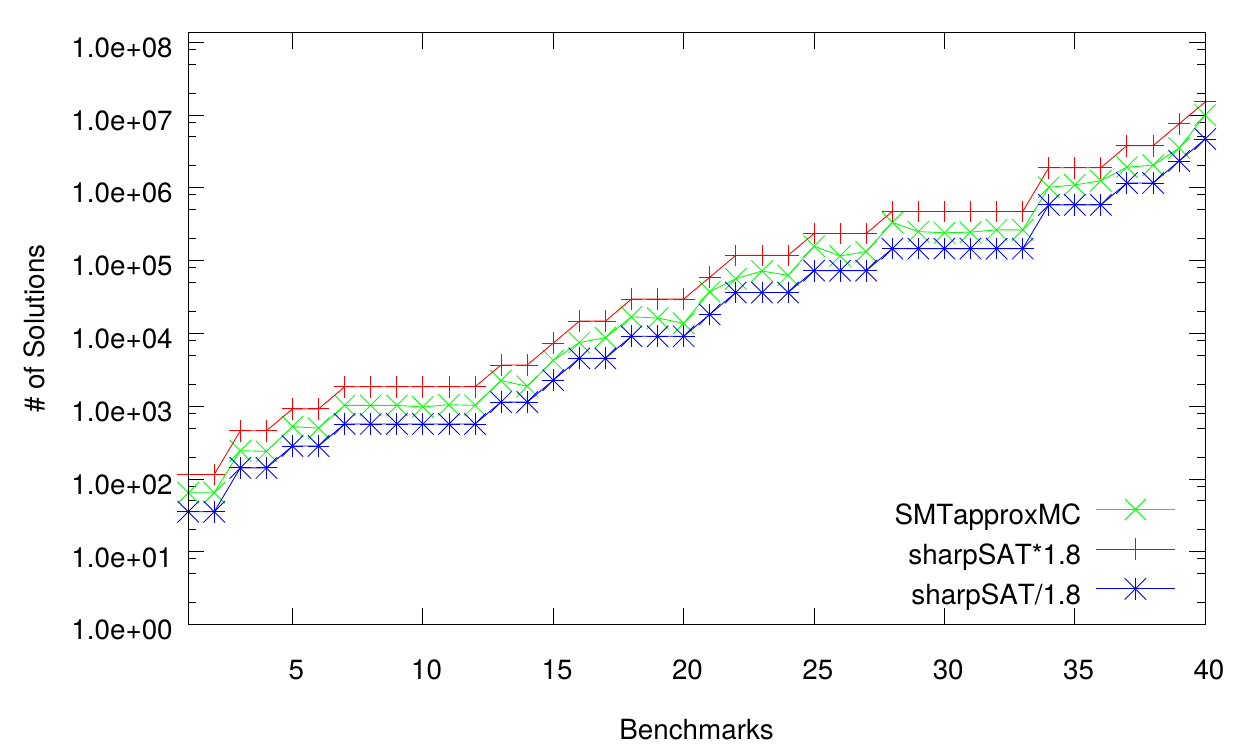}
\caption{Quality of counts computed by {\SMTApproxMC} vis-a-vis exact counts}
\label{fig:quality}
\end{figure}

\section{Conclusions and Future Work}\label{sec:concl}
Hashing-based model counting has emerged as a promising approach for
probabilistic inference on graphical models. While real-world examples
naturally have word-level constraints, state-of-the-art approximate
model counters effectively reduce the problem to propositional model
counting due to lack of non-bit-level hash functions. In this work, we
presented, ${\HH}_{SMT}$, a word-level hash function and used it to
build {\SMTApproxMC}, an approximate word-level model counter.  Our
experiments show that {\SMTApproxMC} can significantly outperform
techniques based on bit-level hashing.

Our study also presents interesting directions for future work.
For example, %
adapting {\SMTApproxMC} to be aware of {\SMT} solving strategies,
and augmenting {\SMT} solving strategies to efficiently reason about
hash functions used in counting, are exciting directions of future work. %

Our work goes beyond serving as a replacement for other approximate
counting techniques. {\SMTApproxMC} can also be viewed as an efficient
building block for more sophisticated inference
algorithms~\cite{BRGD15}.  The development of {\SMT} solvers has so
far been primarily driven by the verification and static analysis
communities. Our work hints that probabilistic inference could well be
another driver for {\SMT} solver technology development.

\subsection*{Acknowledgements}
We thank Daniel Kroening for sharing his valuable insights on SMT solvers during the early stages of this project and Amit Bhatia for comments on early drafts of the paper. 
 This work was supported in part by NSF grants IIS-1527668, CNS 1049862, CCF-1139011, 
	by NSF Expeditions in Computing project "ExCAPE: Expeditions in Computer
	Augmented Program Engineering", by BSF grant 9800096, by a gift from
	Intel, by a grant from the Board of Research in Nuclear Sciences, India,  Data Analysis and Visualization Cyberinfrastructure funded by NSF under grant OCI-0959097.
\fontsize{9.5pt}{10.5pt} \selectfont
\section*{Appendix}
In this section, we provide proofs of various results stated
previously.  Our proofs borrow key ideas
from~\cite{Bellare00,CMV13b,Gomes-Sampling}; however, there are
non-trivial adaptations specific to our work.  We also provide
extended versions of the experimental results reported in
Section~\ref{sec:expts}.

\setcounter{theorem}{0}
\setcounter{lemma}{0}
\section*{Detailed Proofs}\label{sec:detailedProofs}
Let $\mc{D}$ denote $\left(\Z_{p_0}\right)^{c_0} \times
\left(\Z_{p_1}\right)^{c_1} \times \cdots
\left(\Z_{p_{q-1}}\right)^{c_{q-1}}$, where
$\prod_{j=0}^{q-1}p_j^{c_j} < 2^{n.k}$.  Let $C$ denote the vector
$(c_0, c_1, \ldots c_{q-1})$.
\begin{lemma} 
For every $\X \in \{0,1\}^{n.k}$ and every $\alpha \in \mc{D}$, $\prob[h(\X) = \alpha \mid h \xleftarrow{R} {\HH}_{SMT}(n,k,C)] = \prod_{i=0}^{|C|-1}{p_i}^{-c_i}$
\end{lemma}
\begin{proof} 

 Let $h_r$, the $r^{th}$ component of $h$, for $r \le
 \left(\sum_{j=0}^{|C|-1} c_j\right)$, be given by
 $\left(\sum_{m=0}^{n.2^j-1} a_{m}^{(j)}* \X_{m}^{(j)} +
 b^{(j)}\right)\mod p_{j}$, where $\left(\sum_{i=0}^{j-1} c_i\right) <
 r \le \left(\sum_{i=0}^{j} c_i\right)$, and the $a_{m}^{(j)}$s and
 $b^{(j)}$ are randomly and independently chosen elements of
 $\Z_{p_j}$, represented as words of width $\ceil{\log_2 p_j}$.  Let
 ${\HH}^{(j)}$ denote the family of hash functions of the form
 $\left(\sum_{m=0}^{n.2^j-1} u_{m}^{(j)}* \X_{m}^{(j)} +
 v^{(j)}\right)\mod p_{j}$, where $u_{m}^{(j)}$ and $v^{(j)}$ are
 elements of $\Z_{p_j}$. We use $\alpha_r$ to denote the $r$th
 component of $\alpha$. For every choice of $\X$, $a_{m}^{(j)}$s and
 $\alpha_r$, there is exactly one $b^{(j)}$ such that $h_r(\X) =
 \alpha_r$. Therefore, $\prob[h_r(\X) = \alpha_r | h_r \xleftarrow{R}
   {\HH}^{(j)}] = p_i^{-1}$.  

Recall that every hash function $h$ in ${\HH}_{SMT}(n,k,C)$ is a
$\left(\sum_{j=0}^{q-1} c_j\right)$-tuple of hash functions.  Since
$h$ is chosen uniformly at random from ${\HH}_{SMT}(n,k,C)$, the
$\left(\sum_{j=0}^{q-1} c_j\right)$ components of $h$ are effectively
chosen randomly and independently of each other. Therefore,
$\prob[h(\X) = \alpha \mid h \xleftarrow{R} {\HH}_{SMT}(n,k,C)] =
\prod_{i=0}^{|C|-1}{p_i}^{-c_i}$
\end{proof}

\begin{theorem}
For every $\alpha_1, \alpha_2 \in \mc{D}$ and every distinct 
$\X_1,\X_2 \in \{0,1\}^{n.k}$, $\prob[(h(\X_1) = \alpha_1 \wedge h(\X_2) = \alpha_2 )  \mid 
h \xleftarrow{R}  {\HH}_{SMT}(n,k,C)] = \prod_{i=0}^{|C|-1}({p_i})^{-2.c_i}$. 
Therefore, $ {\HH}_{SMT}(n,k,C)$ is pairwise independent.
\end{theorem}

\begin{proof}
We know that $\prob[(h(\X_1) = \alpha_1 \wedge h(\X_2) = \alpha_2 )] = 
\prob[h(\X_2) = \alpha_2 \mid h(\X_1) = \alpha_1 ] \times \prob[h(\X_1) = 
\alpha_1]$. Theorem~\ref{th-universal} implies that in order to prove pairwise independence of ${\HH}_{SMT}(n, k, C)$, it is sufficient to 
show that $ \prob[h(\X_2) = \alpha_2 \mid h(\X_1) = \alpha_1 ] = \prob[h(\X_2) = \alpha_2]$. 
 
Since $h(\X) = \alpha$ can be viewed as conjunction of
$\left(\sum_{j=0}^{q-1} c_j\right)$ ordered and \emph{independent}
constraints, it is sufficient to prove 2-wise independence for every
ordered constraint. We now prove $2$-wise independence for one of the
ordered constraints below.  Since the proof for the other ordered
constraints can be obtained in exactly the same way, we omit their
proofs.

We formulate a new hash function based on the first constraint as
$g(\X) = ( \left(\sum_{m=0}^{n.2^j-1} a_{m}^{(0)}* \X_{m}^{(0)} +
b^{(0)}\right)\mod p_{0}$, where the $a_{m}^{(0)}$'s and $b^{(0)}$ are randomly and
independently chosen elements of $\Z_{p_0}$, represented as words of
width $\ceil{\log_2 p_0}$. It is sufficient to show
that $g(\X)$ is 2-universal.  This can be formally stated as $
\prob[g(\X_2) = \alpha_{2,0} \mid g(\X_1) = \alpha_{1,0} ] =
\prob[g(\X_2) = \alpha_{2,0}]$, where $\alpha_{2,0}, \alpha_{1,0}$ are
the $0^{th}$ components of $\alpha_2$ and $\alpha_1$ respectively. We
consider two cases based on linear independence of $\X_1$
and $\X_2$.

\begin{itemize}

\item \textbf{Case 1:} $\X_1$ and $\X_2$ are linearly
  dependent. Without loss of generality, let $\X_1 = (0,0,0,\ldots 0)$
  and $\X_2 = (r_1,0,0,\ldots 0)$ for some $r_1 \in \Z_{p_0}$,
  represented as a word.  From $g(\X_1)$ we can deduce $b^{(0)}$.
  However for $g(\X_2) = \alpha_{2,0}$ we require $a_{1}^{(0)}*r_1 +
  b^{(0)} = \alpha_{2,0} \mod p_0$.  Using Fermat's Little Theorem, we
  know that there exists a unique $a_{1}^{(0)}$ for every $r_1$ that
  satisfies the above equation.  Therefore, therefore $\prob[g(\X_2) =
    \alpha_{2,0} | g(\X_1) = \alpha_{1,0} ] = \prob[g(\X_2) =
    \alpha_{2,0}]$ $= \frac{1}{p_0}$. \\

\item \textbf{Case 2:} $\X_1$ and $\X_2$ are linearly
  independent. Since $2^k < p_0$, every component of $\X_1$ and $\X_2$
  (i.e. an element of $\{0,1\}^k$) can be treated as an element of
  $\Z_{p_0}$.  The space $\{0,1\}^{n.k}$ can therefore be thought of
  as lying within the vector space $\left(\Z_{p_0}\right)^{n}$, and
  any $\X \in \{0,1\}^{n.k}$ can be written as a linear combination of
  the set of basis vectors over $\left(\Z_{p_0}\right)^{n}$.  It is
  therefore sufficient to prove pairwise independence when $\X_1$ and
  $\X_2$ are basis vectors.  Without loss of generality, let $\X_1 =
  (r_1,0,0,\ldots 0)$ and $\X_2 = (0,r_2,0,0,\ldots 0)$ for some
  $r_1,r_2 \in \Z_{p_0}$. From $g(\X_1)$, we can deduce
  $\left(a_{1}^{(0)}*r_1 + b^{(0)} = \alpha_{1,0}\right) \mod p_0$.
  But since $a_{1}^{(0)}$ is randomly chosen, therefore $\prob[g(\X_2)
    = \alpha_{2,0} \mid g(\X_1) = \alpha_{1,0} ] = \prob
  [(a_{2}^{(0)}*r_2 +\alpha_{1,0} - a_{1}^{(0)}*r_1 =
    \alpha_{2,0})\mod p_0] = \prob[(a_{2}^{(0)}*r_2 - a_{1}^{(0)}*r_1
    = \alpha_{2,0} - \alpha_{1,0})\mod p_0]$, where $-a$ refers to the
  additive inverse of $a$ in the field $\Z_{p_0}$. Using Fermat's
  Little Theorem, we know that for every choice $a_{1}^{(0)}$ there
  exists a unique $a_{2}^{(0)}$ that satisfies the above requirement,
  given $\alpha_{1,0}$, $\alpha_{2,0}$, $r_1$ and $r_2$. Therefore
  $\prob[g(\X_2) = \alpha_{2,0} \mid g(\X_1) = \alpha_{1,0} ] =
  \frac{1}{p_0} = \prob[g(\X_2) = \alpha_{2,0}]$.
\end{itemize}

\end{proof}

\subsection*{Analysis of {\SMTApproxMC}}
 For a given $h$ and
$\alpha$, we use $R_{F,h,\alpha}$ to denote the set $R_F \cap
h^{-1}(\alpha)$, i.e. the set of solutions of $F$ that map to $\alpha$
under $h$.  Let $\expect[Y]$ and $\var[Y]$ represent expectation and
variance of a random variable $Y$ respectively. The analysis below
focuses on the random variable $|R_{F,h,\alpha}|$ defined for a chosen
$\alpha$. We use $\mu$ to denote the expected value of the
random variable $|R_{F,h,\alpha}|$ whenever $h$ and $\alpha$ are clear
from the context. The following lemma based on pairwise independence
of ${\HH}_{SMT}(n,k,C)$ is key to our analysis.
\begin{lemma}\label{lm:cell-pairwise-variance}
   The random choice of $h$ and $\alpha$ in {\SMTApproxMCCore} ensures that for each $\varepsilon > 0$, we have
   $\prob\left[(1-\frac{\varepsilon}{1+\varepsilon})\mu \leq |R_{F,h,\alpha}| \leq (1+\frac{\varepsilon}{1+\varepsilon}) \mu \right] \geq 1 - \frac{(1+\varepsilon)^2}{\varepsilon^2~~\mu}$, where $\mu = \expect[|R_{F,h,\alpha}|]$  
\end{lemma}

\begin{proof}
For every $y \in \{0, 1\}^{n.k}$ and for every $\alpha \in \prod_{i=0}^{|C|-1} (\Z_{p_i})^{C[i]}$, define an indicator variable $\gamma_{y, \alpha}$ as
follows: $\gamma_{y, \alpha} = 1$ if $h(y) = \alpha$, and
$\gamma_{y,\alpha} = 0$ otherwise.  Let us fix $\alpha$ and $y$ and
choose $h$ uniformly at random from ${\HH}_{SMT}(n,k,C)$.  The 2-wise 
independence ${\HH}_{SMT}(n,k,C)$ implies that for every distinct $y_1, y_2 \in R_F$, the random variables $\gamma_{y_1}, \gamma_{y_2}$ are $2$-wise 
independent. Let $|R_{F,h,\alpha}| = \sum_{y \in R_F} \gamma_{y, \alpha}$, 
$\mu  = \expect\left[|R_{F,h,\alpha}|\right]$ and 
$\var[|R_{F,h,\alpha}|] = \var [\sum_{y \in R_F} \gamma_{y, \alpha}]$. The pairwise independence of $\gamma_{y,\alpha}$ ensures that 
$\var[|R_{F,h,\alpha}|] = \sum_{y \in R_F} \var [\gamma_{y, \alpha}] \leq \mu$. The result then follows from Chebyshev's inequality.
\end{proof}
Let $Y$ be the set returned by $\BoundedSMT(F \wedge (h(\X) = \alpha), \mathrm{pivot})$
where $\mathrm{pivot}$ is as calculated in Algorithm~\ref{alg:SMTMC}.
\begin{lemma}\label{lm:cell-correctness} 
$\prob\left[ (1+\varepsilon)^{-1}|R_F| \leq |Y|  \leq (1+\varepsilon)|R_F| \mid \right.$ $\left. \log (\numCells) + \log(\mathrm{pivot})+1 \leq \log |R_F| \right] \geq 	1-\frac{e^{-3/2}}{2^{\log |R_F|- \log (\numCells) - \log(\mathrm{pivot})+1}}$
\end{lemma}
\begin{proof}
Applying Lemma~\ref{lm:cell-pairwise-variance} with $\frac{\varepsilon}{1+\varepsilon} < \varepsilon$, we have 
$\prob\left[ (1+\varepsilon)^{-1}|R_F| \leq |Y|  \leq (1+\varepsilon)|R_F| \mid \right.$ $\left. \log (\numCells) + \log(\mathrm{pivot})-1 \leq \log |R_F| \right] \geq 	1-\frac{e^{-3/2}}{2^{\log |R_F|- \log (\numCells) - \log(\mathrm{pivot})+1}}$
\end{proof}

\begin{lemma}\label{lm:approx-one-shot}
Let an invocation of {\SMTApproxMCCore} from {\SMTApproxMC} return $m$. Then
$\prob\left[ (1+\varepsilon)^{-1}|R_F| \leq m  \leq (1+\varepsilon)|R_F| \right] \geq 	0.6$
\end{lemma}
\begin{proof}
For notational convenience, we use $(\numCells_l)$ to denote the value of $\numCells$ when $i=l$ in the loop in {\SMTApproxMCCore}. 
As noted earlier, we assume, for some $i=\ell^*$, $\log |R_F| - \log(\numCells_{\ell^*}) +1 = \log(\mathrm{pivot})$. Furthermore, note that for all $i \neq j$ and $\numCells_i > \numCells_j$, $\numCells_i / \numCells_j \geq 2$. Let $F_{l}$ denote the event that $|Y| < \mathrm{pivot} \wedge (|Y| > (1+\varepsilon)|R_F| \vee |Y| < (1+\varepsilon)^{-1}|R_F|)$ for $i = l$. Let $\ell_{1}$ be the value of $i$ such that $\numCells_{\ell_{1}} < \numCells_{\ell^*}/2 \wedge \forall j, \numCells_{j} < \numCells_{\ell^*}/2 \implies \numCells_{\ell_{1}} \geq \numCells_{j} $. Similarly, let $\ell_{2}$ be the value of $i$ such that $\numCells_{\ell_{2}} < \numCells_{\ell^*}/4 \wedge \forall j, \numCells_{j} < \numCells_{\ell^*}/4 \implies \numCells_{\ell_{2}} \geq \numCells_{j} $

  Then, $\forall_{i \mid \numCells_i < \numCells_{\ell^*}/4}$, $F_{i} \subseteq F_{\ell_2}$. Furthermore, the probability of $\prob\left[ (1+\varepsilon)^{-1}|R_F| \leq m  \leq (1+\varepsilon)|R_F| \right]$ is at least $1 - \prob[F_{\ell_2}] - \prob[F_{\ell_1}] - \prob[F_{\ell^*}]$ $= 1 - \frac{e^{-3/2}}{4} - \frac{e^{-3/2}}{2} - e^{-3/2}$ $\geq 0.6 $ .

\end{proof}
Now, we apply standard combinatorial analysis on repetition of probabilistic events and prove that {\SMTApproxMC} is $(\varepsilon, \delta)$ model counter. 

\begin{theorem}
Suppose an invocation of {\SMTApproxMC}$(F, \varepsilon, \delta, k)$
returns $\mathrm{FinalCount}$.  Then $\prob\left[
  (1+\varepsilon)^{-1}|R_F| \leq \mathrm{FinalCount} \leq
  (1+\varepsilon)|R_F| \right] \geq 1-\delta$
\end{theorem}
\begin{proof}

Throughout this proof, we assume that {\SMTApproxMCCore} is invoked $t$
times from {\SMTApproxMC}, where $t = \left\lceil 35\log_2 (3/\delta)
\right\rceil$  in
Section~\ref{sec:algo}).  Referring to the pseudocode of {\SMTApproxMC},
the final count returned by {\SMTApproxMC} is the median of non-$\bot$
counts obtained from the $t$ invocations of {\SMTApproxMCCore}.  Let
$Err$ denote the event that the median is not in
$\left[(1+\varepsilon)^{-1}\cdot |R_F|, (1+\varepsilon)\cdot |R_F|\right]$.  Let
``$\#\mathit{non }\bot = q$'' denote the event that $q$ (out of $t$)
values returned by {\SMTApproxMCCore} are non-$\bot$.  Then,
$\prob\left[Err\right]$ $=$ $\sum_{q=0}^t \prob\left[Err \mid
  \#\mathit{non }\bot = q\right]$ $\cdot$
$\prob\left[\#\mathit{non }\bot = q\right]$.

In order to obtain $\prob\left[Err \mid \#\mathit{non }\bot =
  q\right]$, we define a $0$-$1$ random variable $Z_i$, for $1 \le i
\le t$, as follows.  If the $i^{th}$ invocation of {\SMTApproxMCCore}
returns $c$, and if $c$ is either $\bot$ or a non-$\bot$ value that
does not lie in the interval $[(1+\varepsilon)^{-1}\cdot |R_F|,
  (1+\varepsilon)\cdot |R_F|]$, we set $Z_i$ to 1; otherwise, we set
it to $0$.  From Lemma~\ref{lm:approx-one-shot}, $\prob\left[Z_i =
  1\right] = p < 0.4$.  If $Z$ denotes $\sum_{i=1}^t Z_i$, a necessary
(but not sufficient) condition for event $Err$ to occur, given that
$q$ non-$\bot$s were returned by {\SMTApproxMCCore}, is $Z \ge
(t-q+\lceil q/2\rceil)$.  To see why this is so, note that $t-q$
invocations of {\SMTApproxMCCore} must return $\bot$.  In addition, at
least $\lceil q/2 \rceil$ of the remaining $q$ invocations must return
values outside the desired interval. To simplify the exposition, let
$q$ be an even integer.  A more careful analysis removes this
restriction and results in an additional constant scaling factor for
$\prob\left[Err\right]$.  With our simplifying assumption,
$\prob\left[Err \mid \#\mathit{non }\bot = q\right] \le \prob[Z \ge (t
  - q + q/2)]$ $=\eta(t, t-q/2, p)$.  Since $\eta(t, m, p)$ is a
decreasing function of $m$ and since $q/2 \le t-q/2 \le t$, we have
$\prob\left[Err \mid \#\mathit{non }\bot = q\right] \le \eta(t, t/2,
p)$.  If $p < 1/2$, it is easy to verify that $\eta(t, t/2, p)$ is an
increasing function of $p$.  In our case, $p < 0.4$; hence,
$\prob\left[Err \mid \#\mathit{non }\bot = q\right] \le \eta(t, t/2,
0.4)$.

It follows from above that $\prob\left[ Err \right]$ $=$
$\sum_{q=0}^t$ $\prob\left[Err \mid \#\mathit{non }\bot = q\right]$
$\cdot\prob\left[\#\mathit{non }\bot = q\right]$ $\le$ $\eta(t, t/2,
0.4)\cdot$ $\sum_{q=0}^t \prob\left[\#\mathit{non }\bot = q\right]$
$=$ $\eta(t, t/2, 0.4)$.  Since $\binom{t}{t/2} \ge \binom{t}{k}$ for
all $t/2 \le k \le t$, and since $\binom{t}{t/2} \le 2^t$, we have
$\eta(t, t/2, 0.4)$ $=$ $\sum_{k=t/2}^{t} \binom{t}{k} (0.4)^{k}
(0.6)^{t-k}$ $\le$ $\binom{t}{t/2} \sum_{k=t/2}^t (0.4)^k (0.6)^{t-k}$
$\le 2^t \sum_{k=t/2}^t (0.6)^t (0.4/0.6)^{k}$ $\le 2^t \cdot 3 \cdot (0.6 \times 0.4)^{t/2}$ $\le 3\cdot(0.98)^t$.
Since $t = \left\lceil 35\log_2 (3/\delta) \right\rceil$, it follows
that $\prob\left[Err\right] \le \delta$.
\end{proof}
\begin{theorem}
{\SMTApproxMC}$(F, \varepsilon, \delta, k)$ runs in time polynomial in $|F|$,
$1/\varepsilon$ and $\log_2(1/\delta)$ relative to an {\NP}-oracle.%
\end{theorem}
\begin{proof}
Referring to the pseudocode for {\SMTApproxMC}, lines~\ref{line:weightmc-init-start}--~\ref{line:weightmc-init-end} take time
no more than a polynomial in $\log_2(1/\delta)$ and $1/\varepsilon$.
The repeat-until loop in lines~\ref{line:weightmc-loop-start}--~\ref{line:weightmc-loop-end}  is repeated $t = \left\lceil
35\log_2(3/\delta)\right\rceil$ times. The time taken for each
iteration is dominated by the time taken by {\SMTApproxMCCore}.  Finally,
computing the median in line~\ref{line:weightmc-median} takes time linear in $t$.  The proof
is therefore completed by showing that {\SMTApproxMCCore} takes time
polynomial in $|F|$ and $1/\varepsilon$ relative to the {\SAT} oracle.

Referring to the pseudocode for {\SMTApproxMCCore}, we find that
{\BoundedSMT} is called $\mathcal{O}(|F|)$ times.  Each such call can
be implemented by at most $\mathit{pivot}+1$ calls to a {\NP} oracle ({\SMT} solver in case),
and takes time polynomial in $|F|$ and $\mathit{pivot}+1$ relative to
the oracle.  Since $\mathit{pivot}+1$ is in
$\mathcal{O}(1/\varepsilon^2)$, the number of calls to the {\NP}
oracle, and the total time taken by all calls to {\BoundedSMT} in each
invocation of {\SMTApproxMCCore} is a polynomial in $|F|$ and
$1/\varepsilon$ relative to the oracle.  The random choices in lines~\ref{line:SMTMCCore-choose-hash} and~\ref{line:SMTMCCore-choose-alpha} of {\SMTApproxMCCore} can be implemented in time polynomial in $n.k$
(hence, in $|F|$) if we have access to a source of random bits.
Constructing $F \wedge (h(\X) = \alpha)$ in line~\ref{line:SMTMCCore-call-BoundedWeightSAT} can
also be done in time polynomial in $|F|$.
\end{proof}

\section*{Detailed Experimental Results}
The Table~\ref{tab:quality-table-extended} represents the counts corresponding to benchmarks in Figure~\ref{fig:quality}. Column 1 lists the ID for every benchmark, which corresponds to the position on the x-axis in Figure~\ref{fig:quality}. Column 2 lists the name of every benchmark. The exact count computed by {\sharpSATTool} on bit-blasted versions of the benchmarks and column 4 lists counts computed by {\SMTApproxMC}. Similar to the observation based on Figure~\ref{fig:quality}, the Table~\ref{tab:quality-table-extended} clearly demonstrates that the counts computed by {\SMTApproxMC} are very close to the exact counts.

Table~\ref{tab:performance-extended} is an extended version of Table~\ref{tab:performance-comparison}. Similar to Table~\ref{tab:performance-comparison},  column
1 gives the benchmark name, column 2 gives the sum of widths of
all variables, column 3 lists the number of variables ($\mathrm{numVars}$) for
each corresponding width ($\mathrm{w}$) in the format $\{\mathrm{w}:\mathrm{numVars}\}$. To
indicate the complexity of the input formula, we present the number of
operations in the original {\SMT} formula in column 4. The runtimes
for {\SMTApproxMC} and {\CDM} are presented in columns 5 and column 6
respectively.  We use ``--" to denote timeout after $12$ hours.

 \begin{table*}
 \centering

  \caption{Comparison of exact counts vs counts returned by {\SMTApproxMC}}\label{tab:quality-table-extended}

  \begin{tabular}{ |c|c|c|c| }
    
    \hline
    Id & Benchmark & \shortstack{Exact \\Count} & \shortstack{SMTApproxMC\\Count} \\
    \hline
    1 & case127 & 64 & 65 \\ \hline
    2 & case128 & 64 & 65 \\ \hline
    3 & case24 & 256 & 245 \\ \hline
    4 & case29 & 256 & 240 \\ \hline
    5 & case25 & 512 & 525 \\ \hline
    6 & case30 & 512 & 500 \\ \hline
    7 & case28 & 1024 & 1025 \\ \hline
    8 & case33 & 1024 & 1025 \\ \hline
    9 & case27 & 1024 & 1025 \\ \hline
    10 & case32 & 1024 & 975 \\ \hline
    11 & case26 & 1024 & 1050 \\ \hline
    12 & case31 & 1024 & 1025 \\ \hline
    13 & case17 & 2048 & 2250 \\ \hline
    14 & case23 & 2048 & 1875 \\ \hline
    15 & case38 & 4096 & 4250 \\ \hline
    16 & case21 & 8192 & 7500 \\ \hline
    17 & case22 & 8192 & 8750 \\ \hline
    18 & case11 & 16384 & 16875 \\ \hline
    19 & case43 & 16384 & 16250 \\ \hline
    20 & case45 & 16384 & 13750 \\ \hline
    21 & case4 & 32768 & 37500 \\ \hline
    22 & case44 & 65536 & 56250 \\ \hline
    23 & case46 & 65536 & 71875 \\ \hline
    24 & case108 & 65536 & 62500 \\ \hline
    25 & case7 & 131072 & 157216 \\ \hline
    26 & case1 & 131072 & 115625 \\ \hline
    27 & case68 & 131072 & 132651 \\ \hline
    28 & case47 & 262144 & 334084 \\ \hline
    29 & case51 & 262144 & 250563 \\ \hline
    30 & case52 & 262144 & 240737 \\ \hline
    31 & case53 & 262144 & 245650 \\ \hline
    32 & case134 & 262144 & 264196 \\ \hline
    33 & case137 & 262144 & 264196 \\ \hline
    34 & case56 & 1048576 & 1015625 \\ \hline
    35 & case54 & 1048576 & 1093750 \\ \hline
    36 & case109 & 1048576 & 1250000 \\ \hline
    37 & case100 & 2097152 & 1915421 \\ \hline
    38 & case101 & 2097152 & 2047519 \\ \hline
    39 & case2 & 4194304 & 3515625 \\ \hline
    40 & case8 & 8388608 & 9938999 \\ \hline
  \end{tabular}

 \end{table*}

\onecolumn
\begin{longtable}{|c|c|c|c|c|c|}
  	\caption{Extended Runtime performance of {\SMTApproxMC} vis-a-vis {\CDM} for a subset of benchmarks.}
\label{tab:performance-extended} \\

  \hline
    Benchmark & Total Bits & Variable Types & \# of Operations & \shortstack{SMTApproxMC\\ time(s)} & \shortstack{CDM\\ time(s)} \\
    \hline
\endfirsthead

  \hline
    Benchmark & Total Bits & Variable Types & \# of Operations & \shortstack{SMTApproxMC\\ time(s)} & \shortstack{CDM\\ time(s)} \\
    \hline
\endhead

\hline \multicolumn{6}{r}{Continued on next page}
\endfoot

\hline
\endlastfoot
 	squaring27 & 59 & \{1: 11, 16: 3\} & 10 & -- & 2998.97 \\ \hline
    1159708 & 64 & \{8: 4, 32: 1\} & 12 & 14793.93 & -- \\ \hline
    1159472 & 64 & \{8: 4, 32: 1\} & 8 & 16308.82 & -- \\ \hline
    1159115 & 64 & \{8: 4, 32: 1\} & 12 & 23984.55 & -- \\ \hline
    1159520 & 64 & \{8: 4, 32: 1\} & 1388 & 114.53 & 155.09 \\ \hline
    1160300 & 64 & \{8: 4, 32: 1\} & 1183 & 44.02 & 71.16 \\ \hline
    1159005 & 64 & \{8: 4, 32: 1\} & 213 & 28.88 & 105.6 \\ \hline
    1159751 & 64 & \{8: 4, 32: 1\} & 681 & 143.32 & 193.84 \\ \hline
    1159391 & 64 & \{8: 4, 32: 1\} & 681 & 57.03 & 91.62 \\ \hline
    case1 & 17 & \{1: 13, 4: 1\} & 13 & 17.89 & 65.12 \\ \hline
    1159870 & 64 & \{8: 4, 32: 1\} & 164 & 17834.09 & 9152.65 \\ \hline
    1160321 & 64 & \{8: 4, 32: 1\} & 10 & 117.99 & 265.67 \\ \hline
    1159914 & 64 & \{8: 4, 32: 1\} & 8 & 230.06 & 276.74 \\ \hline
    1159064 & 64 & \{8: 4, 32: 1\} & 10 & 69.58 & 192.36 \\ \hline
    1160493 & 64 & \{8: 4, 32: 1\} & 8 & 317.31 & 330.47 \\ \hline
    1159197 & 64 & \{8: 4, 32: 1\} & 8 & 83.22 & 176.23 \\ \hline
    1160487 & 64 & \{8: 4, 32: 1\} & 10 & 74.92 & 149.44 \\ \hline
    1159606 & 64 & \{8: 4, 32: 1\} & 686 & 431.23 & 287.85 \\ \hline
    case100 & 22 & \{1: 6, 16: 1\} & 8 & 32.62 & 89.69 \\ \hline
    1160397 & 64 & \{8: 4, 32: 1\} & 70 & 126.08 & 172.24 \\ \hline
    1160475 & 64 & \{8: 4, 32: 1\} & 67 & 265.58 & 211.16 \\ \hline
    case108 & 24 & \{1: 20, 4: 1\} & 7 & 37.33 & 100.2 \\ \hline
    case101 & 22 & \{1: 6, 16: 1\} & 12 & 44.74 & 90 \\ \hline
    1159244 & 64 & \{8: 4, 32: 1\} & 1474 & 408.63 & 273.57 \\ \hline
    case46 & 20 & \{1: 8, 4: 3\} & 12 & 16.95 & 76.4 \\ \hline
    case44 & 20 & \{1: 8, 4: 3\} & 8 & 13.69 & 72.05 \\ \hline
    case134 & 19 & \{1: 3, 16: 1\} & 8 & 5.36 & 54.22 \\ \hline
    case137 & 19 & \{1: 3, 16: 1\} & 9 & 10.98 & 56.12 \\ \hline
    case68 & 26 & \{8: 3, 1: 2\} & 7 & 34.9 & 67.48 \\ \hline
    case54 & 20 & \{1: 16, 4: 1\} & 8 & 50.73 & 103.91 \\ \hline
    1160365 & 64 & \{8: 4, 32: 1\} & 286 & 98.38 & 99.74 \\ \hline
    1159418 & 32 & \{8: 2, 16: 1\} & 7 & 3.73 & 43.68 \\ \hline
    1160877 & 32 & \{8: 2, 16: 1\} & 8 & 2.57 & 44.01 \\ \hline
    1160988 & 32 & \{8: 2, 16: 1\} & 8 & 4.4 & 44.64 \\ \hline
    1160521 & 32 & \{8: 2, 16: 1\} & 7 & 4.96 & 44.52 \\ \hline
    1159789 & 32 & \{8: 2, 16: 1\} & 13 & 6.35 & 43.09 \\ \hline
    1159117 & 32 & \{8: 2, 16: 1\} & 13 & 5.55 & 43.18 \\ \hline
    1159915 & 32 & \{8: 2, 16: 1\} & 11 & 7.02 & 45.62 \\ \hline
    1160332 & 32 & \{8: 2, 16: 1\} & 12 & 3.94 & 44.35 \\ \hline
    1159582 & 32 & \{8: 2, 16: 1\} & 8 & 5.37 & 43.98 \\ \hline
    1160530 & 32 & \{8: 2, 16: 1\} & 12 & 2.01 & 43.28 \\ \hline
    1160482 & 64 & \{8: 4, 32: 1\} & 36 & 153.99 & 120.55 \\ \hline
    1159564 & 32 & \{8: 2, 16: 1\} & 12 & 7.36 & 41.77 \\ \hline
    1159990 & 64 & \{8: 4, 32: 1\} & 34 & 71.17 & 97.25 \\ \hline
    case7 & 18 & \{1: 10, 8: 1\} & 12 & 17.93 & 51.96 \\ \hline
    case56 & 20 & \{1: 16, 4: 1\} & 12 & 41.54 & 109.3 \\ \hline
    case43 & 15 & \{1: 11, 4: 1\} & 12 & 8.6 & 37.63 \\ \hline
    case45 & 15 & \{1: 11, 4: 1\} & 12 & 9.3 & 35.77 \\ \hline
    case53 & 19 & \{1: 7, 8: 1, 4: 1\} & 9 & 53.66 & 69.96 \\ \hline
    case4 & 16 & \{1: 12, 4: 1\} & 12 & 8.42 & 35.49 \\ \hline
    1160438 & 64 & \{8: 4, 32: 1\} & 2366 & 199.08 & 141.84 \\ \hline
    case109 & 29 & \{1: 21, 4: 2\} & 12 & 171.51 & 179.98 \\ \hline
    case38 & 13 & \{1: 9, 4: 1\} & 7 & 6.21 & 30.27 \\ \hline
    case11 & 15 & \{1: 11, 4: 1\} & 8 & 7.26 & 33.75 \\ \hline
    1158973 & 64 & \{8: 4, 32: 1\} & 94 & 366.6 & 270.17 \\ \hline
    case22 & 14 & \{1: 10, 4: 1\} & 12 & 5.46 & 26.03 \\ \hline
    case21 & 14 & \{1: 10, 4: 1\} & 12 & 5.57 & 24.59 \\ \hline
    case52 & 19 & \{1: 7, 8: 1, 4: 1\} & 9 & 45.1 & 70.72 \\ \hline
    case23 & 12 & \{1: 8, 4: 1\} & 11 & 2.29 & 12.84 \\ \hline
    case51 & 19 & \{1: 7, 8: 1, 4: 1\} & 12 & 40 & 67.22 \\ \hline
    case17 & 12 & \{1: 8, 4: 1\} & 12 & 2.75 & 11.09 \\ \hline
    case33 & 11 & \{1: 7, 4: 1\} & 12 & 1.7 & 9.66 \\ \hline
    case30 & 13 & \{1: 5, 4: 2\} & 13 & 1.41 & 8.69 \\ \hline
    case28 & 11 & \{1: 7, 4: 1\} & 12 & 1.66 & 8.73 \\ \hline
    case25 & 13 & \{1: 5, 4: 2\} & 12 & 1.39 & 8.27 \\ \hline
    case27 & 11 & \{1: 7, 4: 1\} & 12 & 1.69 & 8.57 \\ \hline
    case26 & 11 & \{1: 7, 4: 1\} & 12 & 1.68 & 8.35 \\ \hline
    case32 & 11 & \{1: 7, 4: 1\} & 12 & 1.46 & 8.16 \\ \hline
    case31 & 11 & \{1: 7, 4: 1\} & 12 & 1.64 & 7.64 \\ \hline
    case29 & 12 & \{1: 4, 4: 2\} & 8 & 0.67 & 5.16 \\ \hline
    case24 & 12 & \{1: 4, 4: 2\} & 12 & 0.77 & 4.94 \\ \hline
    1160335 & 64 & \{8: 4, 32: 1\} & 216 & 0.31 & 0.54 \\ \hline
    1159940 & 64 & \{8: 4, 32: 1\} & 94 & 0.17 & 0.04 \\ \hline
    1159690 & 32 & \{8: 2, 16: 1\} & 8 & 0.12 & 0.04 \\ \hline
    1160481 & 32 & \{8: 2, 16: 1\} & 12 & 0.13 & 0.03 \\ \hline
    1159611 & 64 & \{8: 4, 32: 1\} & 73 & 0.2 & 0.09 \\ \hline
    1161180 & 32 & \{8: 2, 16: 1\} & 12 & 0.11 & 0.04 \\ \hline
    1160849 & 32 & \{8: 2, 16: 1\} & 7 & 0.1 & 0.03 \\ \hline
    1159790 & 64 & \{8: 4, 32: 1\} & 113 & 0.15 & 0.04 \\ \hline
    1160315 & 64 & \{8: 4, 32: 1\} & 102 & 0.17 & 0.04 \\ \hline
    1159720 & 64 & \{8: 4, 32: 1\} & 102 & 0.17 & 0.05 \\ \hline
    1159881 & 64 & \{8: 4, 32: 1\} & 102 & 0.16 & 0.04 \\ \hline
    1159766 & 64 & \{8: 4, 32: 1\} & 73 & 0.15 & 0.03 \\ \hline
    1160220 & 64 & \{8: 4, 32: 1\} & 681 & 0.17 & 0.03 \\ \hline
    1159353 & 64 & \{8: 4, 32: 1\} & 113 & 0.16 & 0.04 \\ \hline
    1160223 & 64 & \{8: 4, 32: 1\} & 102 & 0.17 & 0.04 \\ \hline
    1159683 & 64 & \{8: 4, 32: 1\} & 102 & 0.17 & 0.03 \\ \hline
    1159702 & 64 & \{8: 4, 32: 1\} & 102 & 0.19 & 0.04 \\ \hline
    1160378 & 64 & \{8: 4, 32: 1\} & 476 & 0.17 & 0.04 \\ \hline
    1159183 & 64 & \{8: 4, 32: 1\} & 172 & 0.17 & 0.03 \\ \hline
    1159747 & 64 & \{8: 4, 32: 1\} & 322 & 0.18 & 0.03 \\ \hline
    1159808 & 64 & \{8: 4, 32: 1\} & 539 & 0.17 & 0.03 \\ \hline
    1159849 & 64 & \{8: 4, 32: 1\} & 322 & 0.18 & 0.03 \\ \hline
    1159449 & 64 & \{8: 4, 32: 1\} & 540 & 0.3 & 0.05 \\ \hline
    case47 & 26 & \{1: 6, 8: 2, 4: 1\} & 11 & 81.5 & 80.25 \\ \hline
    case2 & 24 & \{1: 20, 4: 1\} & 10 & 273.91 & 194.33 \\ \hline
   
    1159239 & 64 & \{8: 4, 32: 1\} & 238 & 1159.32 & 449.21 \\ \hline
    case8 & 24 & \{1: 12, 8: 1, 4: 1\} & 8 & 433.2 & 147.35 \\ \hline
    1159936 & 64 & \{8: 4, 32: 1\} & 238 & 5835.35 & 1359.9 \\ \hline
	squaring51 & 40 & \{1: 32, 4: 2\} & 7 & 3285.52 & 607.22 \\ \hline
	 1159431 & 64 & \{8: 4, 32: 1\} & 12 & 36406.4 & -- \\ \hline
	    1160191 & 64 & \{8: 4, 32: 1\} & 12 & 40166.1 & -- \\ \hline
\end{longtable}
\end{document}